\documentclass[12pt]{article}
\usepackage{amsmath, amsthm}
\usepackage{url}
\usepackage{graphicx,psfrag,epsf}
\usepackage{enumerate}
\usepackage{natbib}
\usepackage{verbatim}

\setcitestyle{numbers,open={[},close={]}}

\usepackage[symbol]{footmisc}
\usepackage{epstopdf,gensymb}
\usepackage{graphics}
\usepackage{graphicx,color}

\usepackage{amsfonts, animate,subfigure}
\usepackage{amssymb}
\usepackage[margin=1in]{geometry}
\RequirePackage[colorlinks,citecolor=blue,urlcolor=blue]{hyperref}
\usepackage{todonotes}
\usepackage{setspace}
\usepackage{indentfirst}
\usepackage{float}
\usepackage{amsfonts}
\usepackage{bm,pbox}
\usepackage{multirow} 
\usepackage{color}
\usepackage[margin=1in]{geometry}
\usepackage{framed}
\usepackage{amssymb,amsthm,amsfonts,amsbsy,latexsym,amsmath}
\usepackage{enumerate}

\newtheorem{theorem}{Theorem}
\newtheorem{lemma}[theorem]{Lemma}
\newtheorem{proposition}[theorem]{Proposition}

\newtheorem{definition}[theorem]{Definition}

\renewcommand{\geq}{\geqslant}

\def\qed{ \hfill $\blacksquare$}  
\begin{document}
	\def\spacingset#1{\renewcommand{\baselinestretch}%
	{#1}\small\normalsize} \spacingset{1}
	\title{Topic Supervised Non-negative Matrix Factorization}
	
	\author{Kelsey MacMillan\footnote{Master of Analytics Program, 
 University of San Francisco. San Francisco, CA 94117. \url{kjmacmillan@usfca.edu}} $\ $ and James D. Wilson\footnote{ \noindent Department of Mathematics and Statistics, University of San Francisco. San Francisco, CA 94117 \url{jdwilson4@usfca.edu}} }
		
	  \maketitle

\begin{abstract}
	{\footnotesize Topic models have been extensively used to organize and interpret the contents of large, unstructured corpora of text documents. Although topic models often perform well on traditional training vs. test set evaluations, it is often the case that the results of a topic model do not align with human interpretation. This interpretability fallacy is largely due to the unsupervised nature of topic models, which prohibits any user guidance on the results of a model. In this paper, we introduce a semi-supervised method called topic supervised non-negative matrix factorization (TS-NMF) that enables the user to provide labeled example documents to promote the discovery of more meaningful semantic structure of a corpus. In this way, the results of TS-NMF better match the intuition and desired labeling of the user. The core of TS-NMF relies on solving a non-convex optimization problem for which we derive an iterative algorithm that is shown to be monotonic and convergent to a local optimum. We demonstrate the practical utility of TS-NMF on the Reuters and PubMed corpora, and find that TS-NMF is especially useful for conceptual or broad topics, where topic key terms are not well understood. Although finding an optimal latent structure for the data is not a primary objective of the proposed approach, we find that TS-NMF achieves higher weighted Jaccard similarity scores than the contemporary methods, (unsupervised) NMF and latent Dirichlet allocation, at supervision rates as low as 10\% to 20\%. }
\end{abstract}

\noindent%
\newpage
\spacingset{1.5} 

\section{Introduction}



Large, unstructured corpora of text data are generated from an expanding array of sources: online journal publications, news articles, blog posts, Twitter feeds, and Facebook mentions, to name a few. As a result, the need for computational methods to organize and interpret the structure of such corpora has become increasingly apparent. {Topic models} are a popular family of methods used to discover underlying semantic structure of a corpus of documents by identifying and quantifying the importance of representative topics, or themes, throughout the documents. Topic modeling techniques like non-negative matrix factorization (NMF) \cite{xu2003document} and latent Dirichlet allocation (LDA) \cite{blei2012probabilistic, blei2002latent, blei2003latent}, for example, have been widely adopted over the past two decades and have witnessed great success. 

Despite the accomplishments of topic models over the years, these techniques still face a major challenge: \emph{human interpretation}. To put this into context, consider the most common practice in which a user assesses a topic model. First, a test collection of documents with manually labeled topics is held out. A topic model is then trained on the remaining documents, and evaluated on how closely the topics discovered on the training data match those in the test set. As pointed out by \cite{blei2012probabilistic, tealeaves}, evaluation of topic models in this traditional training-test set manner often lead to results that are weakly correlated with human judgment. This is not surprising since the assessment strategy itself generally does not incentivize human interpretation. To provide a concrete example consider the well-studied Reuters corpus from \cite{nltk} that contains 10,788 labeled documents. The following is a document in the  corpus that has been labeled with the topic \emph{sugar}:

\begin{itemize}
\item[]{\it The outcome of today's European Community (EC) white sugar tender is extremely difficult to predict after last week's substantial award of 102,350 tonnes at the highest ever rebate of 46.864 European currency units (Ecus) per 100 kilos, traders said.} \end{itemize}

We ran NMF and LDA on the full corpus. The two closest matching topics to the true label \emph{sugar} were: NMF - \emph{tonne, export, shipment} and \emph{sugar, trader, european}; LDA - \emph{european, french, tonne}, and \emph{sugar, trader, said}. Each of these methods do indeed output a topic containing the true label \emph{sugar}; however, without further analysis it is a challenging task to identify that \emph{sugar} was the desired label of the document. We further analyze the Reuters corpus in Sections 4 and 5. 


In this paper we introduce a novel topic model, known as Topic Supervised NMF (TS-NMF), that dramatically improves the interpretability of contemporary topic models. TS-NMF is a semi-supervised topic model that enables the user to (i) provide examples of documents labeled with known topics and (ii) constrain the topic representation of the corpus to align with the labeled examples. Supervision is formulated as requiring that certain documents either contain a subset of known topics with non-zero strength or explicitly do not contain the identified known topics. By providing known topics to the model rather than allowing the model to generate an arbitrary latent structure, the interpretation of documents containing those known topics will be readily understood. TS-NMF relies on the minimization of a non-convex optimization function. We describe a fast iterative algorithm to approximate this minimization, and prove that the algorithm converges to a local minimum. We apply TS-NMF to two data sets, the Reuters news article corpus \cite{nltk} and the MEDLINE/PubMed abstracts corpus \cite{downloadpubmed}, and assess the utility of our method through a comparison with contemporary topic models, including unsupervised NMF and LDA.

\section{Related Work}
A corpus of $n$ documents and $t$ terms can be represented by a $n \times t$ term-document matrix $V$ containing non-negative entries $v_{i,j} \geq 0$ that quantify the importance of term $j$ in document $i$. The choice of weights $v_{i,j}$ is dependent upon the application, but is typically calculated using some function of term frequency (TF) and inverse document frequency (IDF) (see \cite{salton1988term} for a review). Mathematically, topic models are mappings of $V$ to a lower dimensional representation of $V$ involving the $d < n, t$ {topics} describing the documents. Existing topic modeling approaches generally fall into two classes of methods: {matrix decomposition} methods, which seek a low dimensional representation of $V$ through a factorization into two or more low-rank matrices, and {probabilistic topic modeling} methods, which seek a generative statistical model for $V$. Here, we describe each class of methods in more detail, paying special attention to the works that are most closely related to TS-NMF.


TS-NMF can be viewed as a semi-supervised generalization of NMF. In the case that no supervision is provided, TS-NMF indeed reduces to NMF. NMF is a powerful matrix decomposition technique that has been applied to a variety of areas, including speech denoising, microarray analysis, collaborative filtering, and computer vision \cite{brunet2004metagenes, lee1999learning, shashua2005non, smaragdis2003non}.  More broadly, NMF is useful in settings where the domain of the data is inherently non-negative and where parts-based decompositions are desired. In general, NMF seeks a $n \times d$ non-negative matrix $W$ and a $d \times t$ non-negative matrix $H$ so that $V \approx WH$. The matrices $W$ and $H$ are estimated by minimizing the following objective function:

\begin{equation}\label{eq:NMF_dist} D_{NMF}(W,H) = ||V - WH||_F^2, \quad W \geq 0 \quad H \geq 0, \end{equation}

\noindent where $||\cdot||_F$ is the Frobenius norm. In topic modeling, $W$ and $H$ have a special interpretation: $W_{ij}$ quantifies the relevance of topic $j$ in document $i$, and $H_{ij}$ quantifies the relevance of term $j$ in topic $i$ (see Figure \ref{nmf_toy_figure}). 

\begin{figure}[H]
\centering
  \caption{Illustration of NMF model for topic modeling}
  \includegraphics[width=0.85\linewidth]{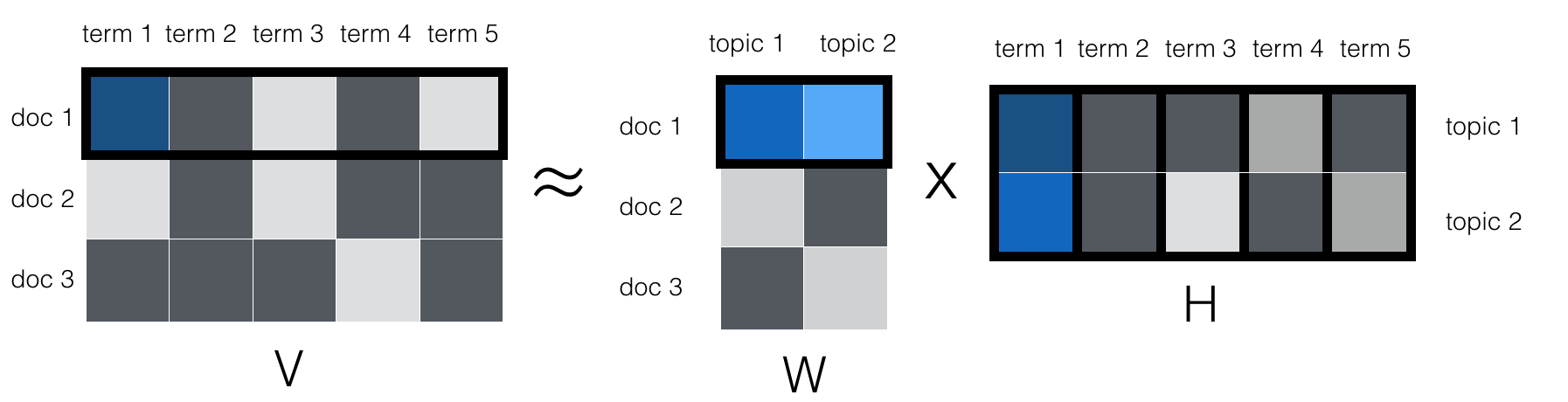}
  \label{nmf_toy_figure}
\end{figure}

Other related extensions of NMF include constrained NMF \cite{constrainednmf}, and semi-supervised NMF \cite{Chen2008}. Constrained NMF assumes that some subset of columns of $V$ have class labels that dictate their location in latent space. In this formulation, one constrains the matrix $H$ so as to enforce the relation that if the $i$th and $j$th columns of $V$ have the same class labels, then the $i$th and $j$th rows of $H$ are equal. Semi-supervised NMF was developed with the objective of identifying the clusters of each column of $V$. With this method, the user can provide pairwise constraints on the columns of $V$, specifying whether they must or cannot be clustered together. The minimization problem in (\ref{eq:NMF_dist}) is reformulated as a non-negative tri-factorization of the similarity matrix that provides the clustering information. As we will see in the next section, each of these methods permits much stronger supervision than TS-NMF in that our proposed method constrains only which of a subset of topics are not allowed to be in some subset of documents.

Latent Semantic Indexing (LSI) \cite{deerwester1990indexing, hofmann1999probabilistic} is a popular matrix decomposition method utilized in topic modeling. LSI seeks a latent representation of $V$ via a singular value decomposition (SVD).  Though not directly applied to topic modeling, \cite{bair2006prediction} and \cite{li2016supervised} introduced supervised versions of the decomposition methods Principal Components Analysis (PCA) and SVD, respectively. In these works, supervision consisted of incorporating auxiliary information for matrix decomposition in the form of a linear model. 

Though not our focus in this paper, probabilistic topic models have been widely applied (see \cite{blei2012probabilistic} for a review). Probabilistic topic models seek a generative statistical model for the matrix $V$. The most prominent of these approaches is latent Dirichlet allocation (LDA) \cite{blei2002latent,blei2003latent}, which models the generation of $V$ as a posterior distribution arising from the probability distributions describing the occurrence of documents over topics as well as the occurrence of topics over terms. LDA takes a Bayesian approach to topic modeling and assumes a Dirichlet prior for the sets of weights describing topics over terms and documents over topics. In this way, the resulting model forms a probability distribution, rather than unconstrained weights on $W$ and $H$. Recent extensions of LDA have considered alternative prior specifications and have begun to explore supervision through apriori topic knowledge \cite{andrzejewski2009latent}.






\section{Topic Supervised Non-negative Matrix Factorization}
Suppose that one supervises $k < < n$ documents and identifies $\ell < < t$ topics that were contained in a subset of the documents. One can supervise the NMF method using this information. This supervision can be represented by the $n \times d$ topic \emph{supervision matrix} $L$. The elements of $L$ are of the following form:

$$L_{ij} = \begin{cases} 1& \text{if topic $j$ is permitted in document $i$} \\ 0 & \text{if topic $j$ is \emph{not} permitted in document $i$} \end{cases}$$

We use the supervision matrix $L$ to constrain the importance weights $W_{ij}$. For all pairs $(i,j)$ such that $L_{ij} = 0$, we enforce that $W_{ij}$ must also be 0. One can view this constraint as requiring the labeled documents to lie within a prescribed subspace of the latent topic space. Let $\circ$ represent the Hadamard product operator. For a term-document matrix $V$ and supervision matrix $L$, TS-NMF seeks matrices $W$ and $H$ that minimize
  
\begin{equation} \label{eq:TS_Distance}
D_{TS}(W,H) = {\lvert \lvert V- (W \circ L) H \rvert \rvert}_{F}^{2}, 
								\quad W \geq 0, \quad H \geq 0.
\end{equation}






\subsection{The TS-NMF Algorithm}

In general the optimization function $D_{TS}$ is not convex in its arguments. Thus we seek an iterative monotonic algorithm that improves in each iteration. We first devise an iterative multiplicative update rule for minimizing (\ref{eq:TS_Distance}), then we prove that our update rules are monotonically non-increasing in $D_{TS}$ and will converge to a local minimum fixed point.

We begin with an equivalent representation of the Frobenius norm of a matrix: $||A||_F^2 = \text{Tr}(A^TA)$, where $\text{Tr}(\cdot)$ represents the trace of a matrix, or the sum of its entries on the diagonal. Using this representation and properties of the trace we are able to simplify the expression in (\ref{eq:TS_Distance}) as follows.

\begin{align}\label{eq:trace}
D_{TS}(W, H) & = \text{Tr}(VV^T) - 2\text{Tr}(V^T(W \circ L)H) + \text{Tr}(H^T(W \circ L)^T(W \circ L) H). 
\end{align}


We apply the method of Lagrange multipliers on (\ref{eq:trace}) with the constraints $W_{ij} \geq 0$ and $H_{ij} \geq 0$. Let $\bm{\alpha} = (\alpha_{ij})$ and  $\bm{\beta} = (\beta_{ij})$ denote the matrices of Lagrange multipliers for W and H, respectively. Then, minimizing $D_{TS}(W, H)$ is equivalent to minimizing the Lagrangian $\mathcal{L}(W, H)$: 

\begin{equation} \label{eq:lagrange} \mathcal{L}(W, H) = D_{TS}(W, H) + \text{Tr}(\bm{\alpha} W^T) + \text{Tr}(\bm{\beta} H^T)\end{equation}

Taking partial derivatives in (\ref{eq:lagrange}) yields

\begin{equation}\label{eq:deriv1}
{\partial \mathcal{L}(W,H)}/{\partial W} = -2(VH^T \circ L) + 2((W \circ L) HH^T) \circ L + \alpha.
\end{equation}

\begin{equation}\label{eq:deriv2}
{\partial \mathcal{L}(W,H)}/{\partial H} = -2V^T(W \circ L) + 2H(W \circ L)^T (W\circ L) + \bm{\beta}.
\end{equation}

Note that minimizing $D_{TS}(W,H)$ can be achieved in an element-wise fashion. With this in mind, we apply the Karush Kuhn Tucker conditions $\alpha_{ij}W_{ij} = 0$ and $\beta_{ij}H_{ij} = 0$ and set the expressions in (\ref{eq:deriv1}) and (\ref{eq:deriv2}) to 0 for each element (i,j) to obtain the following system of equations

\begin{equation}\label{eq:system1}[H (W\circ L)^T (W \circ L)]_{rj} H_{rj} - [V^T(W\circ L)]_{rj} H_{rj} = 0\end{equation}

\begin{equation}\label{eq:system2}[((W \circ L) H H^T) \circ L]_{ir} W_{ir} - [VH^T \circ L]_{ir} W_{ir} = 0 \end{equation}

Solving equations (\ref{eq:system1}) and (\ref{eq:system2}) leads to the following iterative algorithm that sequentially updates the entries of $H$ and $W$. 

\begin{framed}
\underline{\bf Algorithm: TS-NMF Multiplicative Update Rules}

\begin{itemize}
\item[] {\bf Initialize} $H^0$, $W^0$, $t = 0$ 
\item[] {\bf Loop}
\begin{itemize}
\item[] {\bf Set} $H = H^t$, $W = W^t$
\item[] {\bf For} $r = 1, \ldots, d$, $j = 1, \ldots, t$, and $i = 1, \ldots, n$: 
\item[] \begin{equation} \label{eq:3}
H_{rj}^{t+1} = H_{rj}^{t}\frac{[(W \circ L)^TV]_{rj}}{[(W \circ L)^T(W \circ L)H]_{rj}}
\end{equation}

\item[] \begin{equation} \label{eq:4}
W_{ir}^{t+1} = W_{ir}^{t}\frac{[(VH^T) \circ L]_{ir}}{[((W \circ L)HH^T) \circ L]_{ir}}
\end{equation}
\item[] {\bf If} $H^{t+1} = H^{t}$ and $W^{t+1} = W^{t}$, {\bf break}
\item[] {\bf Else} set $t = t+1$ and  repeat {\bf Loop}
\end{itemize}
\item[] {\bf Return} Fixed points $H^* = H^{t+1}$, $W^* = W^{t+1}$
\end{itemize}
\end{framed}

Initialization of $H^0$ and $W^0$ is arbitrary; however, we follow the suggestion of \cite{initialization} and initialize $H^0$ using the Random Acol method. The entries of $W^0$ are sampled as independent realizations from a uniform random variable on the unit interval.

\subsection{Monotonicity of the Algorithm}
We now analyze the convergence and monotonicity of the TS-NMF algorithm by analyzing the update rules (\ref{eq:3}) and (\ref{eq:4}). Our main result shows that the optimization function $D_{TS}(W,H)$ is non-increasing in these updates, and that a fixed point $(W^*, H^*)$ will be a stationary point of the function $D_{TS}(W,H)$. This suggests that the output of the algorithm, $(W^*, H^*)$ will be a local minimum of $D_{TS}(W, H)$. Our main result is made precise in the following theorem.

\begin{theorem} \label{thm:optimality}
The optimization function $D_{TS}(W,H)$ is non-increasing under the update rules (\ref{eq:3}) and (\ref{eq:4}). Furthermore, $D_{TS}(W,H)$ is invariant under these updates if and only if $(W^{t}, H^{t})$ is a stationary point for $D_{TS}(W,H)$. It follows that $(W^*, H^*)$ is a local minimum for $D_{TS}(W,H)$. 
\end{theorem}

The proof of Theorem \ref{thm:optimality} relies on identifying an appropriate auxiliary function for $D_{TS}(W,H)$ and proving that the update rules (\ref{eq:3}) and (\ref{eq:4}) minimize the chosen auxiliary function at each iteration. This proof technique was, for example, utilized in \cite{dempster1977maximum} for analyzing convergence properties of the Expectation-Maximization algorithm. We fill in the details below.

Let $F_H(H_{rj})$ denote the part of $D_{TS}(W,H)$ that depends on the element $H_{rj}$, and $F_W(W_{ir})$ the part that depends on $W_{ir}$. Since the update rules in (\ref{eq:3}) and (\ref{eq:4}) are element-wise, it is sufficient to show that $F_H(H_{rj})$ and $F_W(W_{ir})$ are non-increasing in the updates to prove Theorem \ref{thm:optimality}. To show this, we construct an appropriate auxiliary function for $F_H(H_{rj})$ and $F_W(W_{ir})$. An auxiliary function for $F(x)$ is defined as follows:

\begin{definition}
$G(x,x')$ is an auxiliary function for $F(x)$ if $G(x,x') \geq F(x)$ and $G(x,x) = F(x)$.
\end{definition}

The following well-known lemma reveals the importance of auxiliary functions in studying the convergence of iterative algorithms.

\begin{lemma}\label{lemma1}
If $G(x, x')$ is an auxiliary function for the function $F(x)$, then $F$ is non-increasing under the update
\begin{equation}\label{updates} x^{t+1} = \text{argmin}_x G(x, x')\end{equation}
\end{lemma}

An important consequence of Lemma \ref{lemma1} is that $F(x^{t+1}) = F(x^{t})$ if and only if $x^t$ is a local minimum of $G(x, x^t)$. It follows that iterating the update in (\ref{updates}) will converge to a local minimum of $F$. We proceed by identifying auxiliary functions for $F_H(H_{rj})$ and $F_W(W_{ir})$. The next proposition explicitly identifies the auxiliary functions that we need.

\begin{proposition}\label{prop:auxiliary_functions}
Let $F_W(W_{ir})$ denote the part of (\ref{eq:trace}) that depends on $W_{ir}$, and let $F_H(H_{rj})$ be the part of (\ref{eq:trace}) that depends on $H_{rj}$. Define the following two functions:
$$G_H(h, H_{rj}^t) = F_H(H_{rj}^t) + \dfrac{\partial F_H(H_{rj}^t)}{\partial H}(h - H_{rj}^t) + \dfrac{[(W\circ L)^T (W \circ L) H]_{rj}}{H_{rj}^t}(h - H_{rj}^t)^2$$
$$G_W(w, W_{ir}^t) = F_W(W_{ir}^t) + \dfrac{\partial F_H(W_{ir}^t)}{\partial W}(w - W_{ir}^t) + \dfrac{[((W \circ L) HH^T) \circ L]_{ir}}{W_{ir}^t}(w - W_{ir}^t)^2.$$
Then $G_H(h, H_{rj}^t)$ is auxiliary for $F_H(h)$ and $G_W(w, W_{ir}^t)$ is auxiliary for $F_W(w)$. 
\end{proposition}

\begin{proof} It is obvious that $G_H(h, h) = F_H(h)$ and $G_W(w,w) = F_W(w)$. To show that $G_H(H_{rj}^t, h) \geq F_H(H_{rj}^t)$ one can instead compare $G_H(H_{rj}^t, h)$ with the second order Taylor expansion of $F_H(H_{rj}^t)$. Through a little algebra, the result follows. The same argument holds true to show $G_W(W_{ir}^t, w) \geq F_W(W_{ir}^t)$. \end{proof}
Our final proposition relates the desired update rule in (\ref{updates}) to our proposed update rules in (\ref{eq:3}) and (\ref{eq:4}).

\begin{proposition}
Let $G_H(h, H_{rj}^t)$ and $G_W(w, W_{ir}^t)$ be defined as in Proposition (\ref{prop:auxiliary_functions}). Then, 

\begin{equation}\label{min1}
\text{argmin}_h G_H(h, H_{rj}^t) = H_{rj}^{t}\frac{[(W \circ L)^TV]_{rj}}{[(W \circ L)^T(W \circ L)H]_{rj}}
\end{equation}

\begin{equation}\label{min2}
\text{argmin}_w G_W(w, W_{ir}^t) = W_{ir}^{t}\frac{[(VH^T) \circ L]_{ir}}{[((W \circ L)HH^T) \circ L]_{ir}}
\end{equation}
\end{proposition}
\begin{proof}
The function $G_H(h, H_{rj}^t)$ is quadratic in $h$ and has a non-negative second derivative. Furthermore, $G_W(w, W_{ir}^t)$ is quadratic in $w$ and has a non-negative second derivative. Equations (\ref{min1}) and (\ref{min2}) directly follow.
\end{proof}

{\bf Proof of Theorem 1}: 
Recall that minimizing $D_{TS}(W,H)$ can be done in an element-wise fashion, namely by minimizing $F_H(H_{rj})$ and $F_W(W_{ir})$ for all pairs $(r,j)$ and $(i,r)$. It thus suffices to show that $F_H(H_{rj})$ and $F_W(W_{ir})$ are non-increasing in the update rules (\ref{eq:3}) and (\ref{eq:4}), respectively. Proposition \ref{prop:auxiliary_functions} reveals that $G_H(h, H_{rj}^t)$ and $G_W(w, W_{ir}^t)$ are auxiliary functions for $F_H(h)$ and $F_W(w)$, respectively. Thus according to Lemma \ref{lemma1}, $D_{TS}(W,H)$ will be non-increasing under the following updates:
$$H_{rj}^{t+1} = \text{argmin}_h G_H(h, H_{rj}^t); \qquad W_{ir}^{t+1} = \text{argmin}_w G_W(w, W_{ir}^t)$$
Applying (\ref{min1}) and (\ref{min2}) to the above updates, we obtain our desired result. \qed 

\subsection{Error Weighting}
In some applications, it may be of interest to incorporate a weighting scheme that captures ``importances'' of the documents in $V$. To do so, one can introduce a $n \times t$ weighting matrix $E$ whose entries represent the importance weight of the respective entries in $V$. The aim of TS-NMF then becomes to identify matrices $W \geq 0$ and $H \geq 0$ that minimize $D_{TSW}(W,H) = {\lvert \lvert ((V- (W \circ L) \times H)) \circ E \rvert \rvert}_{F}^{2}$. Although error weighting isn't required in TS-NMF, we find empirically that it can improve results by penalizing poor representations of labeled documents more heavily than unlabeled documents. For the results that follow we run TS-NMF on $D_{TSW}(W,H)$ with $E$ having row $i$ set to 1 if document $i$ was not supervised, and to the inverse supervision rate if document $i$ was supervised. We further discuss the weighting procedure and how to adapt the update rules in (\ref{eq:3}) and (\ref{eq:4}) to $D_{TSW}$ in the supplemental material.


\section{Experiments}

To assess the efficacy of TS-NMF, we study two datasets: the Reuters data set provided as part of the NLTK Python module \cite{nltk}, and abstracts from NIH's MEDLINE/PubMed citation record \cite{downloadpubmed}. The 
Reuters data set contains 10,788 labeled news documents, and each document often had multiple labels. We removed any Reuters news documents that had fewer than 250 characters, which left 8294 documents. From the MEDLINE/PubMed data set, we use a subset of 15,320 abstracts as the corpus and apply PubMed controlled MeSH keywords as labels. Abstracts and keywords are filtered such that each keyword label is attributed to at least 100 abstracts and each abstract has at least one keyword label. In each corpus, tokens were created by splitting the text into 1-grams with NLTK's Treebank Word tokenizer and then lemmatized using NLTK's Wordnet lemmatizer. Documents were encoded into a document-term matrix using TFIDF with a vocabulary of 2000 and L2-normalization of the term frequencies within each document. We also removed stopwords and any tokens with fewer than 3 characters long. In the Reuters data set, the true number of labels is 90, and in the PubMed dataset it is 265, so we set $k=90$ and $k=265$, respectively. 

We ran TS-NMF on both the Reuters and PubMed data sets, and compared discovered topics with LDA and unsupervised NMF. For TS-NMF we applied a range of supervision according to the supervision rate, which is defined as the proportion of documents included in the labeled set. We provide more detailed summaries of the data in the supplement.

\subsection{Evaluation Metrics}
Topic models are evaluated by comparing the similarity between the topics found by a model and the ground-truth labels provided in the Reuters dataset and by the PubMed MeSH keywords. Let $W$ and $\widetilde{W}$ be the identified and true topic - document matrices, respectively. To calculate the similarity between $W$ and $\widetilde{W}$, one first calculates a similarity matrix that expresses the Jaccard distrance between the rows of each topic - document matrix, separately. Let $\mathbf{W}_j$ be the $j$th row of $W$. Then, the Jaccard match between the $i$th and $j$th row is defined as:
\begin{equation} \label{eq:jaccard}
J(\mathbf{W}_i,\mathbf{W}_j) = \frac{\sum_{k=1}^d min(W_{ik},W_{jk})}{\sum_{k=1}^d max(W_{ik},W_{jk})}.
\end{equation}
The Kuhn-Munkres matching algorithm \cite{hungarian} is subsequently applied to the similarity matrices for $W$ and $\widetilde{W}$ to identify the optimal one-to-one mapping from discovered topics to known labels in such a way that the sum of all similarities between the mapped topics and true labels are maximized. A topic is considered "resolved" when it has a similarity score greater than 0.1.

\section{Results}
To gain practical insight into the interpretability of the models, we first provide some examples of labeled Reuters articles in Table \ref{tab:examples}. Document 1 of Table \ref{tab:examples} shows a case where all topic models provide reasonable interpretations. They all correctly identify `yen' and `dollar' as key topic defining terms and the meaning of those terms is clear. In the case of Document 2 all models correctly identify `sugar' as a key term, but the unsupervised methods (LDA and NMF) contain some spurious terms, such as `export', `european', and `said' that muddle interpretation. The benefit of topic supervision becomes particularly evident in Document 3. Notably, the topics found via NMF and LDA don't have a clear human interpretation. TS-NMF captures the terms `production' and `industrial' even with only a 0.2 supervision rate. The same holds true for higher supervision rates. 

\begin{table}
\caption{Example Reuters results from NMF, LDA, and TS-NMF at 0.2, 0.5, and 0.8 supervision rate. Results reveal that TS-NMF identify interpretable topics that closely match the true topic labels.}
\label{tab:examples}
\resizebox{\textwidth}{!}{ 
    \begin{tabular}{|p{3.5cm}|p{3.5cm}|p{3.5cm}|p{3.5cm}|p{3.5cm}|}
    \hline
    \multicolumn{5}{|p{17.75cm}|}{{\bf Document 1}: \emph{The Bank of Japan bought a small amount of dollars shortly after the opening at around 145.30 yen, dealers said. The central bank intervened as a medium - sized trading house sold dollars, putting pressure on the U.S. Currency, they said.} {\bf Labels}: dollar; money-foreign exchange; yen}\\
    \hline
    {\bf TS-NMF (0.2)} & {\bf TS-NMF (0.5)} & {\bf TS-NMF (0.8)} & {\bf NMF} & {\bf LDA} \\
    \hline
    dollar, yen, dealer (0.34) & dollar, yen, dealer (0.41) & dollar, yen, dealer (0.33) & yen, tokyo (0.34); dollar, dealer, currency (0.41) & dollar, yen, said (0.86) \\
    \hline
    \multicolumn{5}{p{17.5cm}}{}\\
    \hline
    \multicolumn{5}{|p{19cm}|}{{\bf Document 2}: \emph{The outcome of today's European Community (EC) white sugar tender is extremely difficult to predict after last week's substantial award of 102,350 tonnes at the highest ever rebate of 46.864 European currency units (Ecus) per 100 kilos, traders said.} {\bf Labels}: sugar }\\
    \hline 
    {\bf TS-NMF (0.2)} & {\bf TS-NMF (0.5)} & {\bf TS-NMF (0.8)} & {\bf NMF} & {\bf LDA} \\ \hline
sugar, tonne, white (0.29) & barley, tonne, ecus (0.19); sugar, tonne, white (0.27) & barley, tonne, ecus (0.15); sugar, tonne, white (0.19) & tonne, export, shipment (0.20); sugar, trader, european (0.37) & european, french, tonne (0.30); said, market, analyst (0.30); sugar,trader,said (0.37) \\
\hline
\multicolumn{5}{p{17.8cm}}{}\\
\hline
\multicolumn{5}{|p{19cm}|}{{\bf Document 3}: \emph{China's industrial output rose 14.1 pct in the first quarter of 1987 against the same 1986 period, the People's Daily said. Its overseas edition said the growth rate, which compares with a target of seven pct for the whole of 1987 was "rather high" but the base in the first quarter of 1986 was on the low side. Industrial output grew 4.4 pct in the first quarter of 1986. It said China's industrial production this year has been
normal but product quality and efficiency need further improvement. It gave no further details.} \qquad \qquad \qquad {\bf Labels}: industrial production index}\\
\hline 
    {\bf TS-NMF (0.2)} & {\bf TS-NMF (0.5)} & {\bf TS-NMF (0.8)} & {\bf NMF} & {\bf LDA} \\ \hline
china, agency, news (0.18); january, industrial, production (0.22); quarter, first, fourth (0.17) & pct, industrial, production (0.35); china, chinese (0.22) & quarter, first, third (0.17); pct, industrial, january (0.30); china, chinese, daily (0.28) & china, production, output (0.35); quarter, first, result (0.27) & pct, growth, year (0.25); pct, january, february (0.30); said, market, analyst (0.18) \\
\hline
    \end{tabular}}

\end{table}

Figure \ref{score_plot} shows the impact of supervision rate on  weighted Jaccard similarity score and topic resolution with greater granularity. Initially, at supervision rates below 20\%, we see rapid improvements in terms of similarity. Then, across a midrange of supervision rates (about 20\% to  70\%) we see a leveling off. Finally, at high supervision rates (greater than 70\%), nearly all remaining topics are resolved. This suggests that there exists an optimal supervision rate regime that balances labeling costs with marginal model improvements.\\

\begin{figure}[ht]
\centering
  \caption{Reuters TS-NMF similarity measures by supervision rate.}
  \includegraphics[width=.825\linewidth]{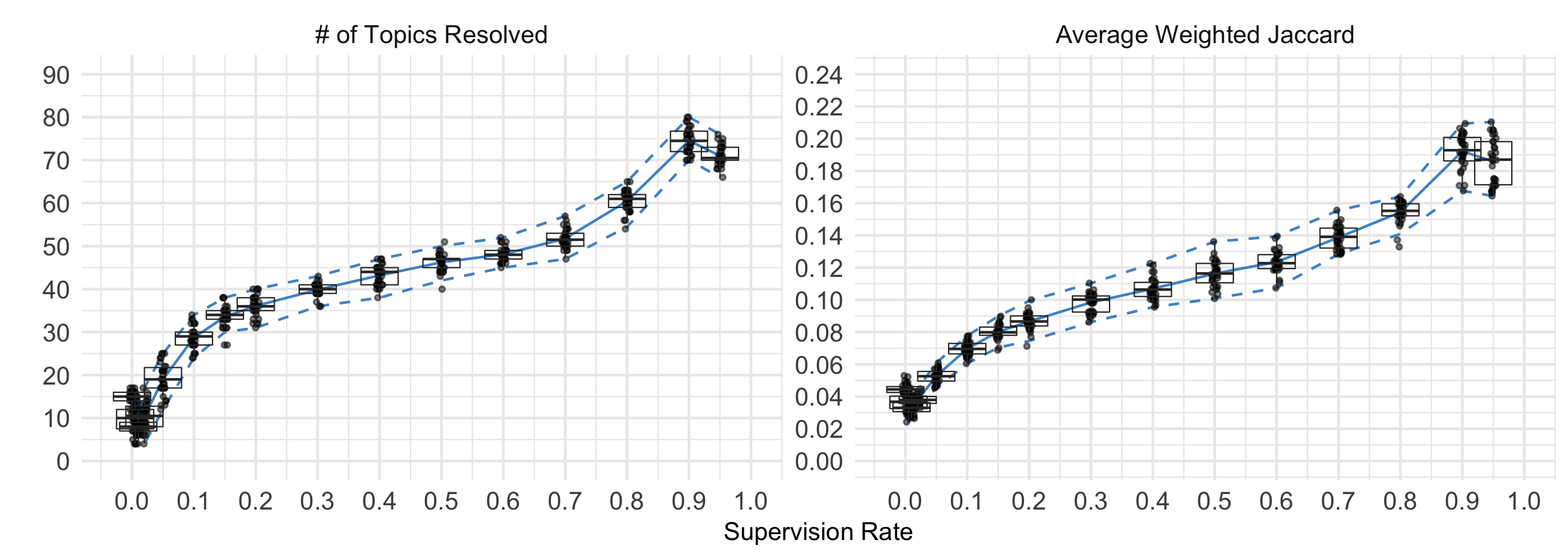} 
  \label{score_plot}
\end{figure}
\vskip -1pc
\begin{figure}[ht]
\centering
  \caption{Comparison of LDA, NMF, and TS-NMF (at 0.05, 0.2, 0.5, 0.8) on Reuters and PubMed.}
  \includegraphics[width=.85\linewidth]{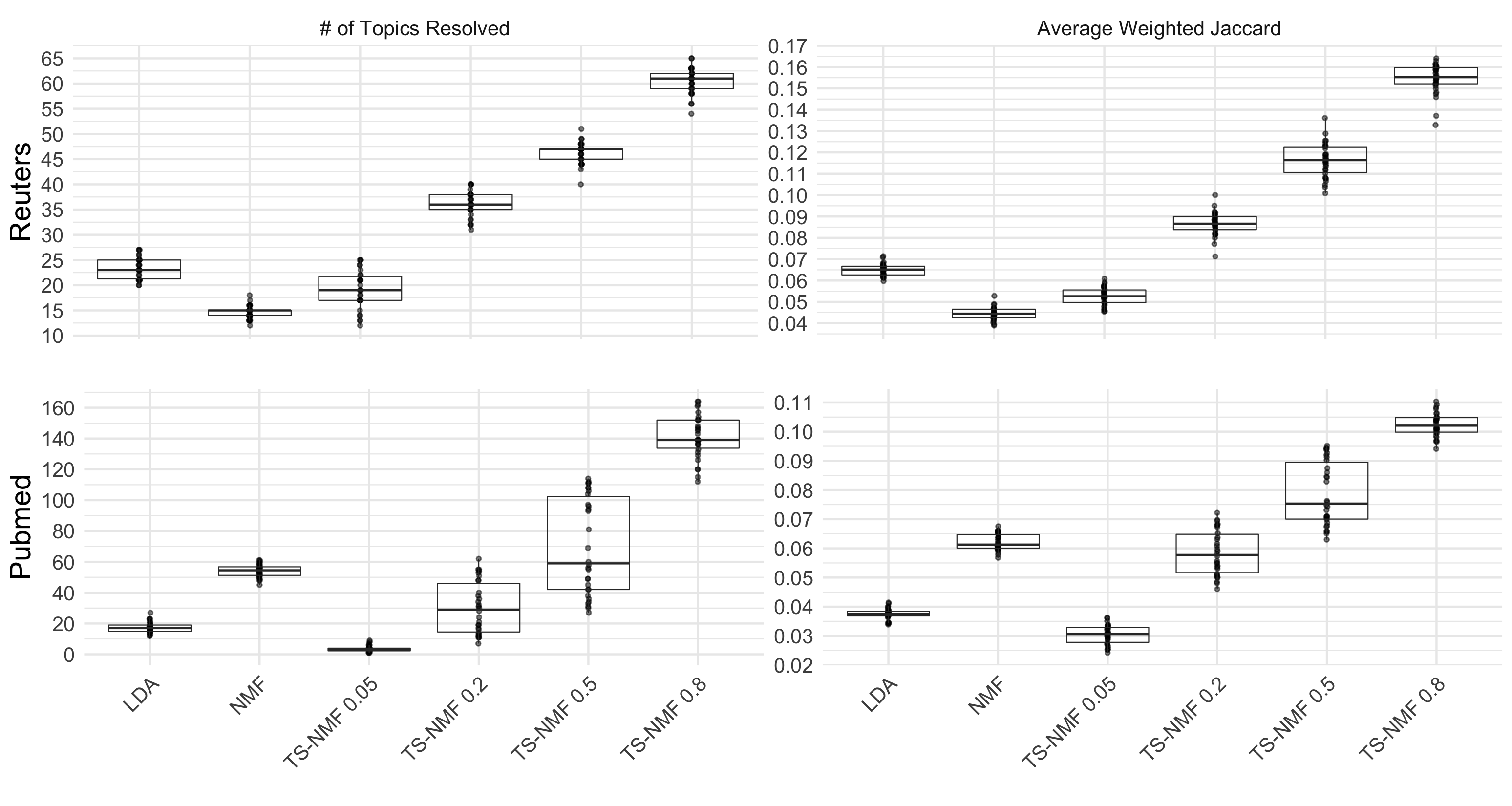}
  \label{compare_plot}
\end{figure}

Figure \ref{compare_plot} compares similarity scores across topic modeling methods for the two chosen corpora. In the case of the Reuters data set, LDA outperforms NMF and 5\% TS-NMF. However, we see significant gains at supervision levels of 20\%. In the case of the PubMed data set, however, LDA performs much worse than NMF. Even more interestingly, NMF outperforms TS-NMF up to a supervision level of about 50\%. Moreover, NMF's results are less variable than the TS-NMF results, implying that the subset of documents included in the labeled set are very important for this corpus. In fact, there appears to be bi-modality in the 50\% TS-NMF results. We note that although the weighted Jaccard similarity score is useful for quantifying how well the model represents the prescribed latent structure of the data set,  it does not say anything about interpretability of the results. Human labeling doesn't necessarily correspond to the best latent structure for the data, which is what the factorization and probabilistic models optimize for. This disagreement between human and machine is manifested by low Jaccard scores and is the crux of the interpretation challenges that our proposed supervision strategy seeks to address.

\section{Discussion}

In this paper we introduced a novel topic modeling method, TS-NMF, that improves human interpretability of topics discovered from large, poorly understood corpora of documents. In such applications, TS-NMF enables the user to incorporate supervision by providing examples of documents with desired topic structure. Although we focused on the formulation of TS-NMF for topic modeling, this method can, with appropriate choice of supervision, be readily generalized to any non-negative matrix decomposition application. We developed an iterative algorithm for TS-NMF based on multiplicative updates and proved the monotonicity of the algorithm and its convergence to a local optimum. Finally, we ran TS-NMF on the Reuters and PubMed corpora and compared the method to state of the art topic modeling methods. We found in both applications that there was a regime of low-to-moderate supervision rates that balanced cost (labeling is expensive) and utility (significant improvement in the model). We have shown that TS-NMF is an effective topic modeling strategy that should be considered in applications when human interpretability is important.



\section*{Appendix}
\appendix

\section{Error Weighting of TS-NMF}

We note that in practice very low supervision rates (few  labeled documents compared to the entire corpus) may result in a representation that "ignores" the labeled data in favor of a factorization with lower loss. This is counter to the interpretability benefit of TS-NMF. To provide a factorization that seeks minimum error around the labeled data, we can weight the ``important''  (minority case) examples more heavily in the loss function.

In this setting, error weighting can be done by introducing a new error weight matrix $E$, which has the same shape as $V$. $E$ has values at row indices of unlabeled documents  equal to $1$ (no weighting) and  value at labeled document indices greater to $1$. Typically, inverse frequency weighting is used, $(number\ of\ documents)/(number\ of\ labeled\ documents)$. However labeled document weights could also be assigned according to  a confidence metric. For instance, one may  have both user labeled documents and expert labeled documents, and  want to penalize the expert labeled documents more heavily.  With error weighting, the loss function for our problem becomes:

\begin{equation}
D_{TSW} = {\lvert \lvert ((V- (W \circ L) \times H)) \circ E \rvert \rvert}_{F}^{2}, 
\quad W \geq 0, \quad H \geq 0
\end{equation}

The corresponding update rules for error-weighted TS-NMF (to replace equations (9) and (10) in the main document) are as follows
\begin{equation}
H_{rj}^{t+1} = H_{rj}^{t}\frac{[(W \circ L)^T(V \circ E)]_{rj}}{[(W \circ L)^T((W \circ L)H \circ E)]_{rj}}
\end{equation}

\begin{equation}
W_{ir}^{t+1} = W_{ir}^{t}\frac{[((V \circ E)H^T) \circ L]_{ir}}{[((((W \circ L)H) \circ E)H^T) \circ L]_{ir}}
\end{equation}

\section{Additional Empirical Analyses}
We compare the weighted versus non-weighted versions of TS-NMF on the Reuters data set across a grid of supervision rates as well. Figure \ref{fig:weighted} illustrates the number of topics resolved and the Jaccard match under each of these methods. Generally, the results show that weighted TS-NMF outperforms unweighted TS-NMF across rates of supervision between 0.2 and 0.7, but that the two methods are comparable for other supervision rates.

\begin{figure}
	\caption{Comparison of weighted and non-weighted TS-NMF methods on the Reuters data set. \label{fig:weighted}}
	\centering
	\includegraphics[width = .8\textwidth]{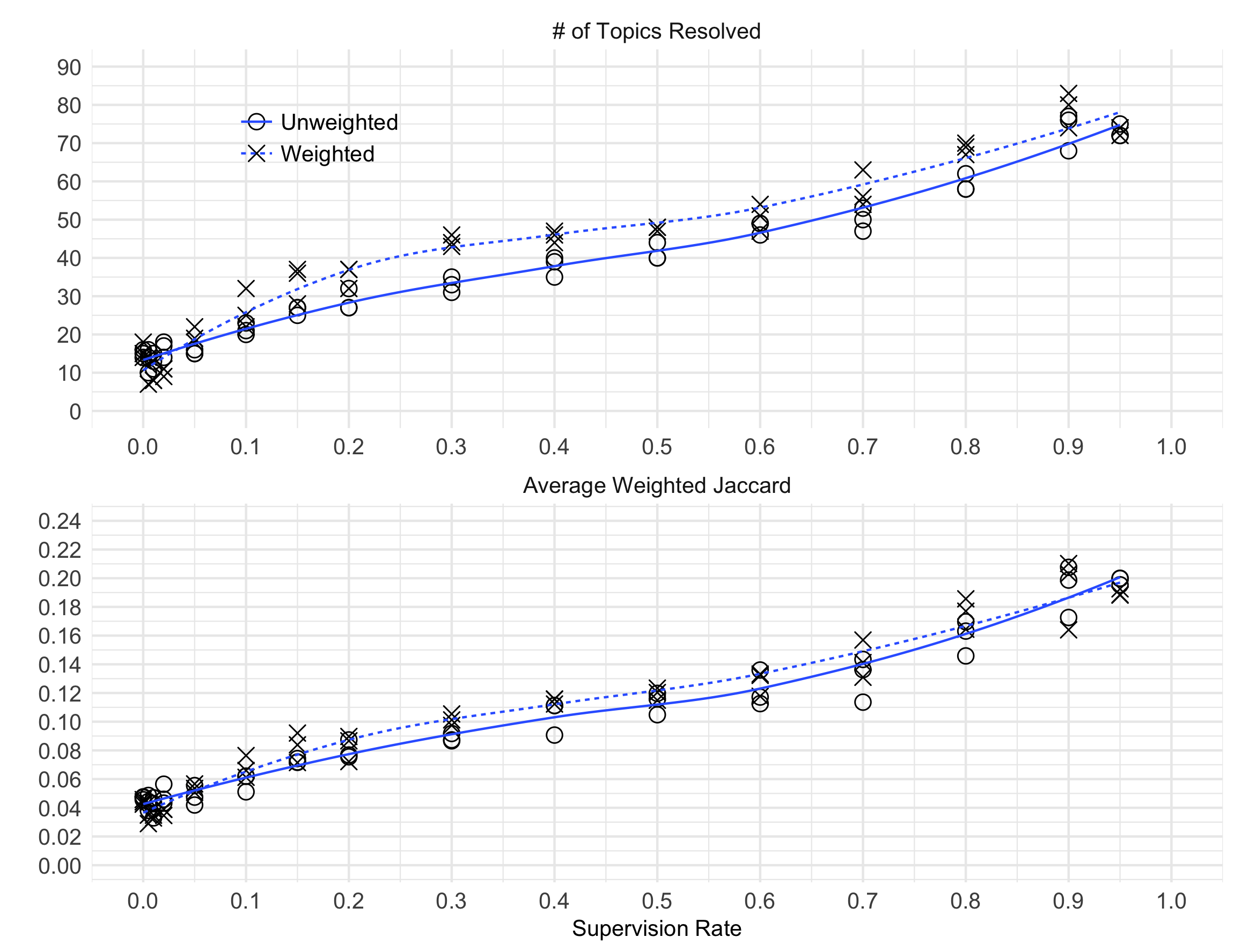}
\end{figure}

It may be of interest to investigate whether the supervision rate that we specify is conflated with topic coverage on either the Reuters or PubMed data set. To test this, we take 1000 random samples of documents at a specified supervision rate from the Reuters data set and calculate the topic coverage, or proportion of topics that are incorporated in the sample of documents.  We plot the relationship between topic coverage and supervision rate in the left plot of Figure \ref{coverage_plot}. This plot suggests that it is not immediately apparent whether the model is more responsive to the amount of topic coverage or the supervision rate. Moreover, because documents to topics is a one to many mapping, it is not clear how one would manipulate topic coverage while holding supervision rate steady. To investigate this dynamic, we took advantage of the fact that each of our runs at a given supervision rate produced a variety of topic coverage rates due the random selection of the labeled document subset. Therefore, we could examine whether a relatively large topic coverage rate at a given supervision rate was predictive of a higher score. The right plot of Figure \ref{coverage_plot} is a plot of topic coverage rate deviation from group mean against score deviation from group mean, where groups are defined by the supervision rate. Based on Figure \ref{coverage_plot}, there is no relationship between topic coverage and model quality at a given supervision rate. 

\begin{figure}
\centering
  \caption{(Left) Topic coverage by supervision rate. (Right) Topic coverage deviation by similarity score deviation within supervision rate groups.}
  \includegraphics[width=0.45\linewidth]{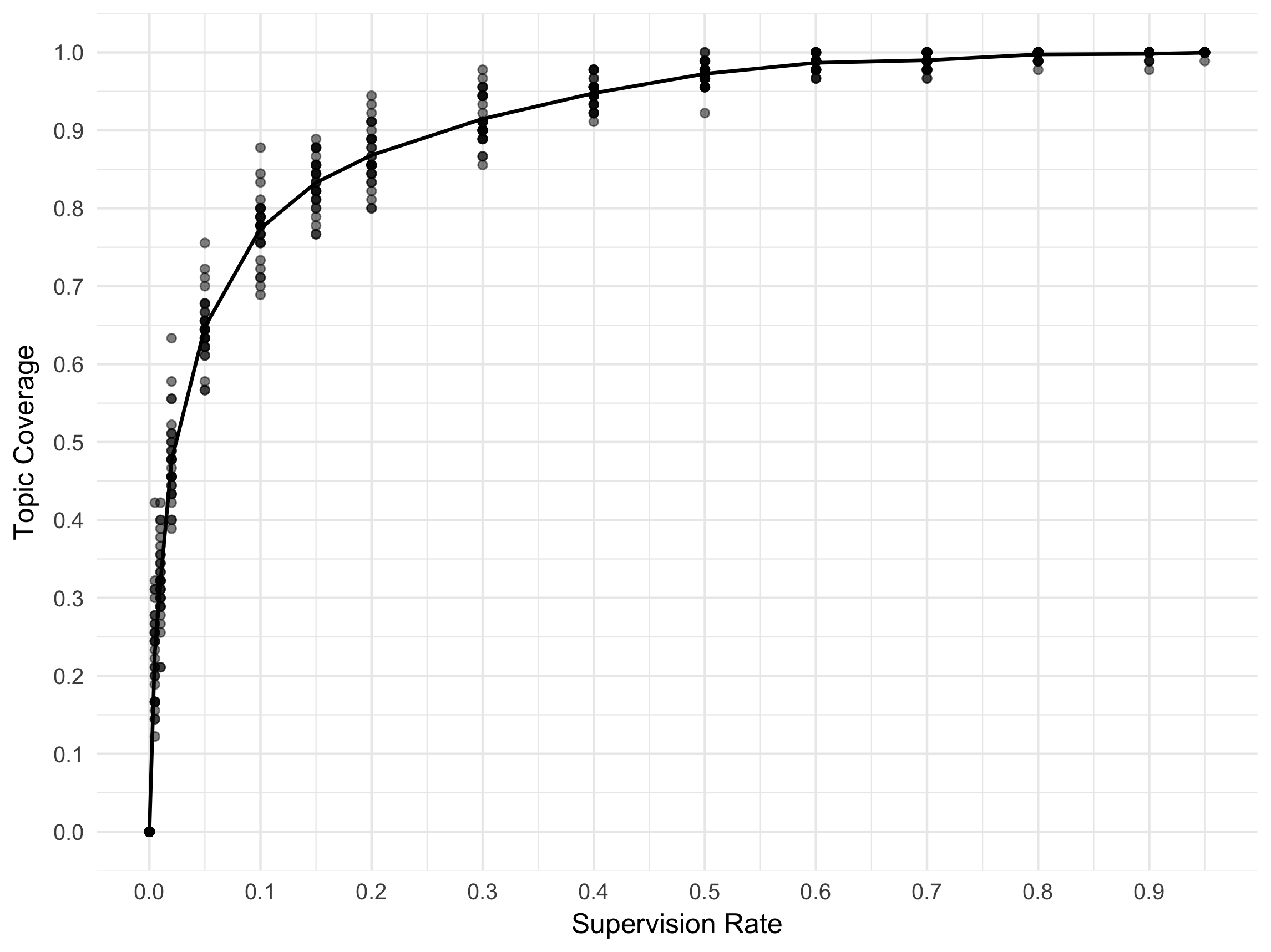} \includegraphics[width=0.45\linewidth]{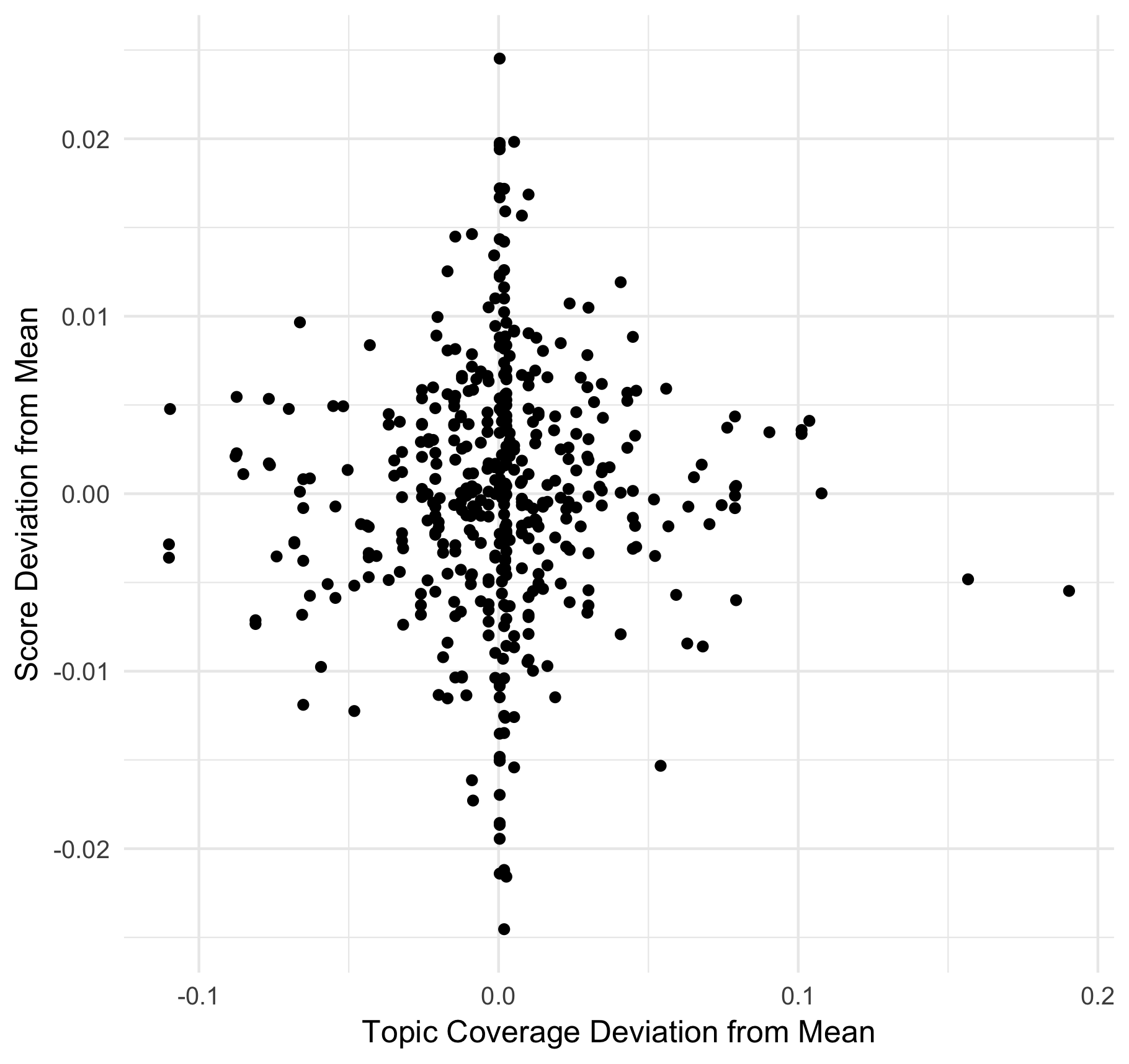}
  \label{coverage_plot}
\end{figure}


\newpage

\small
\bibliographystyle{plain}


\end{document}